\documentclass[letterpaper]{article} 
\usepackage{aaai2026}
\usepackage{times}  
\usepackage{helvet}  
\usepackage{courier}  
\usepackage[hyphens]{url}  
\usepackage{graphicx} 
\usepackage{natbib}  
\usepackage{caption} 
\usepackage{algorithm}
\usepackage{algorithmic}
\usepackage{multirow}
\usepackage{makecell}

\usepackage{amssymb}
\usepackage{amsmath}
\usepackage{amsthm}
\usepackage{xspace}
\usepackage{tcolorbox}
\newcommand{\eg}{e.g.\@\xspace}

\usepackage{booktabs}

\newtheorem{theorem}{Theorem}[section]

\usepackage{listings}
\DeclareCaptionStyle{ruled}{labelfont=normalfont,labelsep=colon,strut=off} 
\floatstyle{ruled}
\newfloat{listing}{tb}{lst}{}
\floatname{listing}{Listing}

\title{UDA: \underline{U}nsupervised \underline{D}ebiasing \underline{A}lignment for Pair-wise LLM-as-a-Judge}

\author{
    Yang Zhang\textsuperscript{\rm 1,\rm 2}\thanks{Work was done when interned at Z.ai.}\equalcontrib, Cunxiang Wang\textsuperscript{\rm 1,\rm 3}\equalcontrib\thanks{Corresponding author.}, Lindong Wu\textsuperscript{\rm 1}, Wenbo Yu\textsuperscript{\rm 1}, \\Yidong Wang\textsuperscript{\rm 1,\rm 2}, Guangsheng Bao\textsuperscript{\rm 4}, Jie Tang\textsuperscript{\rm 3}\footnotemark[3]
}
\affiliations{
    \textsuperscript{\rm 1}ZhipuAI\\
    \textsuperscript{\rm 2}Peking University\\
    \textsuperscript{\rm 3}Tsinghua University\\
    \textsuperscript{\rm 4}Westlake University\\
    zhang360428@stu.pku.edu.cn, wangcunxiang303@gmail.com
}

\usepackage{bibentry}

\begin{document}

\maketitle

\begin{abstract}
Pairwise evaluation of Large Language Models (LLMs) is a common paradigm, but it is prone to preference bias, where judges systematically favor certain outputs, such as their own. This bias leads to inconsistent and skewed rankings across different judges. To address this, we first empirically demonstrate significant and heterogeneous biases in cross-model evaluations. We then propose UDA (Unsupervised Debiasing Alignment), a framework that reduces inter-judge disagreement by dynamically adjusting the Elo rating system. For each pairwise comparison, a compact neural network learns to adaptively set the K-factor and refine win probabilities. Crucially, UDA operates in a fully unsupervised manner, guided solely by the objective of minimizing the dispersion among the Elo trajectories of all judges. This forces an alignment towards a collective consensus, which serves as an unsupervised proxy for a more stable and reproducible evaluation. In addition, we provide theoretical motivation demonstrating how alignment towards a consensus can reduce aggregate system bias. Experiments show that UDA significantly reduces the inter-judge rating standard deviation by up to 63.4\% and improves the average correlation with human judgments by 24.7\%. Notably, UDA elevates the performance of poorly performing judges to achieve parity with high-quality ones, fostering a more robust and reliable evaluation ecosystem. Code and data are available at \url{https://anonymous.4open.science/r/62AB93CD-23B4}.
\end{abstract}

\section{Introduction}
Large Language Models (LLMs) like GPT-4 and Claude-3 now drive progress across Natural Language Processing (NLP), while their daily use in commerce, education, and entertainment seamlessly integrates them into everyday life \citep{ref1zhao2025surveylargelanguagemodels, ref2minaee2025largelanguagemodelssurvey}. Rigorous evaluation is therefore indispensable: it guides model optimization and equips users with an empirical basis for selection, ensuring robust deployment in the wild \citep{ref3chang2023surveyevaluationlargelanguage}.

\begin{figure}[t]
\centering
\includegraphics[width=0.9\columnwidth]{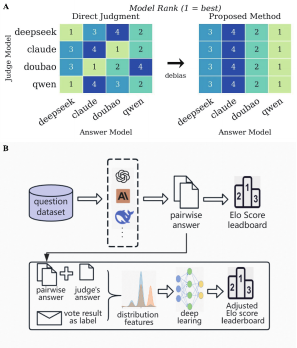}
\caption{\textbf{A.} Some models prefer their own answers, and mitigating this bias from using different LLMs can make the results more accurate. \textbf{B.} Dynamically recalibrate, rather than naively adopts, LLM-judged scores, using consensus among judges as the supervisory label in lieu of human annotation.}
\label{fig0}
\end{figure}
The paradigm of  `LLM-as-a-judge' \citep{ref4zheng2023judgingllmasajudgemtbenchchatbot} delegates the evaluation to a second large model that acts as an autonomous judge: the judge is prompted to score, rank or critique the outputs of the target model without human labels. Yet this very mechanism introduces systematic bias \citep{ref5wang2023largelanguagemodelsfair, ref6thakur2025judgingjudgesevaluatingalignment}. In this paradigm, pairwise judgments consistently outperform pointwise scoring in robustness and are widely adopted, for example, \citeauthor{ref4zheng2023judgingllmasajudgemtbenchchatbot} first established this in the LMSys Chatbot Arena \citep{ref20chiang2024chatbotarenaopenplatform}: GPT-4 achieved markedly higher agreement with human preferences when ranking anonymized responses head-to-head than when assigning absolute scores. Subsequent works  \citep{ref13wang2025improvingllmasajudgeinferencejudgment} corroborate that pairwise ranking reduces intra-model variance; on RewardBench, GPT-4o attains 90.5\% accuracy in pairwise mode versus 88.0\% in pointwise mode, while smaller models trend in the opposite direction. 
Different LLM are built and trained in unique ways, causing them to have specific biases, such as a preference for wordy or stylistically similar answers. This results in varied scores and rankings. Prior works show that score disparities exceeding 30\% among LLMs on identical benchmarks \citep{ref7ye2024justiceprejudicequantifyingbiases}.

Most researches \citep{ref3chang2023surveyevaluationlargelanguage} mainly focus on the improvement of evaluation indicators and methods, but the studies on how to reduce the impact of model preferences during the evaluation process are not in-depth enough. Although recent work calibrates proprietary models on-the-fly via prompt engineering \citep{ref10zhou2024mitigatingbiaslargelanguage}, the practice remains ad-hoc and theoretically thin. Ensemble schemes \citep{ref11wataoka2025selfpreferencebiasllmasajudge} mitigate bias but explode in cost, and Nash-based re-ranking \citep{ref12liu2025reevaluatingopenendedevaluationlarge} scales poorly. In addition, there is relatively little research on how to use deep learning techniques to dynamically adjust the evaluation process, rather than just directly using the evaluation results provided by large models, in order to reduce the bias from different judge models.

To address these limitations, we propose UDA (Unsupervised Debiasing Alignment), a fully automated, annotation-free framework that reconciles heterogeneous LLM judges. UDA retains the anonymity and simplicity of pairwise Elo tournaments but injects a learned, instance-level correction mechanism. Concretely, for each comparison, we (i) calculate the distributional similarity between the two answers, (ii) quantify the judge’s preference by contrasting its own generated response with the two candidates, and (iii) feed these continuous signals into a lightweight neural adapter that outputs an adaptive K-factor and refined win probability expectations. The adapter is trained to pull each judge's scoring trajectory towards the collective consensus of all judges. This consensus, while not a perfect ground truth, serves as a pragmatic, unsupervised signal to mitigate extreme idiosyncratic biases. Because these adjustments are conditioned solely on the response distributions, UDA is model-agnostic and requires no additional fine-tuning.

Bias usually shifts across topics: a judge fair on one task may favor itself on another. To verify generalisation beyond public benchmarks, we evaluate UDA zero-shot on a newly curated \textsc{Human-Annotated Transfer Set} that contains 100 open-ended prompts distinct from the training corpus. 
Without any retraining, UDA still shrinks the inter-judge standard deviation by 63.4\% and boosts correlation with human preferences by 24.7\%, demonstrating robust transfer to unseen domains and judges.

We first quantify the tendency of each judge to over- or under-rate its own response.
On our curated dataset, baseline Elo produces self-scores that deviate from the average score assigned by the remaining judges by as much as $-$21\% (underrating) to +56\% (overrating).
After applying UDA, the same deviations narrow to only $-$15\% to +4\%, indicating that extreme self-preference is largely suppressed.
Under mild assumptions, UDA minimises the \emph{aggregate} bias across all judges (Theorem~1).
However, because the procedure pulls every judge toward the consensus, exceptionally well-calibrated judges may experience a slight drop in individual accuracy; the collective variance reduction dominates this per-model cost.
Despite the small per-model trade-off, UDA improves the average Pearson correlation with human rankings from 0.65 to 0.81 (+25\%), confirming that shrinking inter-judge disagreement also aligns the final scores more closely with human preferences.

To the best of our knowledge, we are the first to reduce pairwise LLM-as-a-judge preference bias in an unsupervised manner.

\section{Related Work}
\subsection{LLM-as-a-Judge}
The LLM-as-a-judge approach has been widely adopted for evaluating LLMs due to its low cost and high efficiency \citep{ref14gu2025surveyllmasajudge}. This method leverages one LLM to assess the outputs of another, providing fine-grained evaluations. For instance, OpenAI's GPT and Anthropic's Claude are commonly utilized as judge models to evaluate the performance of other LLMs. However, this approach may introduce significant bias. Different judge models can produce varying results and rankings, as their specific traits, differences in training data, and subjective evaluation criteria can influence the evaluation results.

\subsection{Elo Rating System}
The Elo rating system \citep{ref18ELOpaper}, traditionally used in chess, has been adapted for evaluating large language models (LLMs). Chatbot Arena \citep{ref20chiang2024chatbotarenaopenplatform} introduced an Elo system that assesses chatbots through pairwise comparisons based on user votes. This method offers a dynamic and interactive approach to evaluation, providing a degree of objectivity and comparability. However, it may be influenced by user preferences and the diversity of evaluation questions. Yet, this system still relies on user feedback and faces challenges in scalability and scenario diversity.  \citet{gray2022usingeloratingmetric} combined the ELO system with traditional metric-based evaluations, integrating model rankings with specific metric analyses for a more comprehensive assessment. Despite these advancements, existing ELO evaluation systems in LLMs still struggle with balancing the weight of different evaluation dimensions and criteria.

\subsection{Bias in LLM-as-a-Judge}
Preference bias in LLMs is a critical evaluation issue \citep{ref11wataoka2025selfpreferencebiasllmasajudge}. It causes LLMs to overrate certain outputs, especially their own output. Studies \citep{ref4zheng2023judgingllmasajudgemtbenchchatbot, ref15xu2024prideprejudicellmamplifies} reveal that some LLMs favor their own generations. This bias may step from the model's training data and architecture. Larger and more capable models often show stronger self-preference. The model's internal probabilities may favor familiar generation styles, and alignment training artifacts could also contribute. \citeauthor{ref16panickssery2024llmevaluatorsrecognizefavor} explored the link between self-preference bias and self-recognition ability. As shown in Figure \ref{fig2}, it reveals heterogeneous LLM preferences.

\begin{figure}[t]
\centering
\includegraphics[width=1.0\columnwidth]{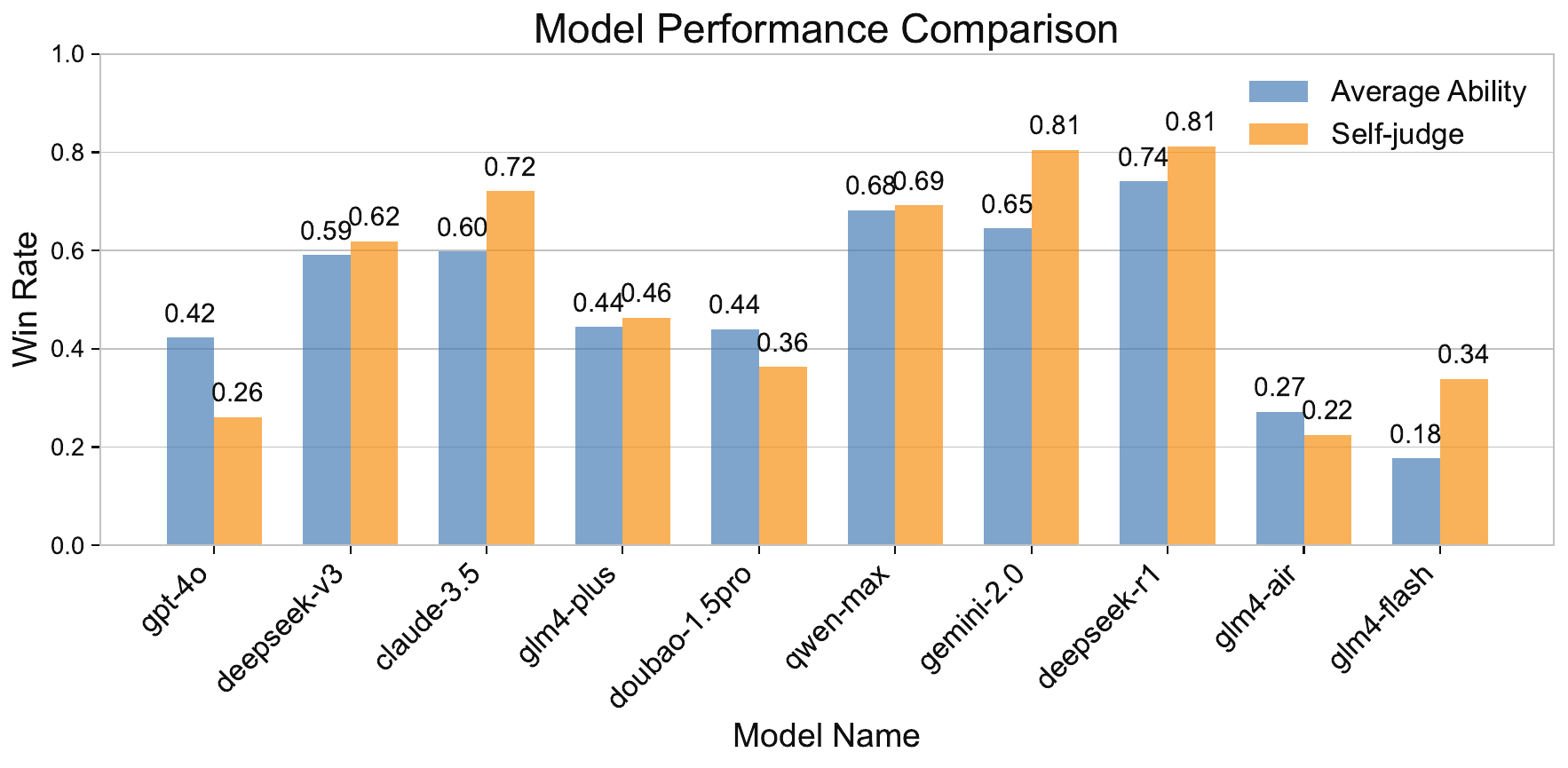} 
\caption{Pairwise answer evaluation experiments were conducted using different large models as judges on the same Arena Hard dataset. The figure shows that LLMs exhibit self-preferential bias: some over-rate their own answers relative to other judges, others under-rate them. On ArenaHard (the depicted dataset) this bias ranges from $-$38\% to +90\%, while on our dataset it spans $-$21\% to +56\%.}
\label{fig2}
\end{figure}

Other biases also exist. Some models show position bias, preferring the first or last option in pairwise comparisons. Model performance is influenced by whether real names or anonymous aliases are used \citep{ref17koo2024benchmarkingcognitivebiaseslarge}. Large models tend to favor their own responses under anonymity, but this self-preference decreases when real names are shown, while smaller models' self-preference increases.

\subsection{Mitigating LLM Judgers' Bias}
\citet{ref10zhou2024mitigatingbiaslargelanguage}  tackled the bias issue in the ``LLM-as-a-judge" approach by quantifying the bias in large models' predictive distributions and then subtracting it to reduce the bias. However, this method has limitations when evaluating close-source llm. Due to the inaccessibility of predictive logic, it is difficult to calculate the distributional differences between pre-trained and SFT models. Consequently, it resorts to prompt engineering to model surface quality and mitigate bias. \citet{ref12liu2025reevaluatingopenendedevaluationlarge} adopted a game-theoretic perspective, framing evaluation as a three-player game and proposing a Nash equilibrium solution. Yet, the equilibrium ratings lack scalability and cannot be directly applied to many real-world large models.
\citet{ref19liu2025amelostableframeworkarenabased} tried to optimize ELo rating systems in arena settings, they  proposed the am-ELO framework, which uses maximum likelihood estimation (MLE) to adjust the Elo probability function by incorporating annotator abilities, enhancing stability and accuracy from an annotation accuracy perspective.

However, these methods either rely on inaccessible model internals, resort to prompt engineering, or heuristics-based reranking that scales poorly. Moreover, these methods overlook the unequal impact of individual comparisons in scoring systems: a weaker model’s loss to a near-equal rival is weighted identically to its loss to a markedly stronger one, an evident flaw that undermines power when sample sizes are small. They also ignore judge-induced distributional biases—e.g., inflated scores arising simply from a preference for responses resembling their own.
In contrast, our approach dynamically reweights each pairwise outcome and incorporates distributional features to correct for biases, using the principle of consensus alignment as a supervisory signal in the absence of human labels.

\section{Method}
\label{sec:UDA}
UDA (Unsupervised Debiasing Alignment) is an annotation-free method that reconciles a set of $M$ heterogeneous LLM judges, $\mathcal{J}$, by dynamically adjusting the Elo update rule for a set of $N$ candidate models, $\mathcal{M}$. Given a prompt $x$, a judge $k \in \mathcal{J}$ expresses its raw binary preference $r_{ij}^{(k)}\in\{0,1\}$ for paired responses $(a_i, a_j)$ from models $i, j \in \mathcal{M}$, which in turn determines the Elo score $R_i^{(k)}\in\mathbb{R}$ for each model under that judge. The key idea is to treat the \textbf{collective consensus ranking} of all judges as a pragmatic, unsupervised target without human labels.

\subsection{Baseline Elo Updating Rule}
Baseline Elo updates judge $k$'s score via
\begin{equation}
  R_i^{(k)}\leftarrow R_i^{(k)} + K\bigl(r_{ij}^{(k)} - \hat p_{ij}^{(k)}\bigr),\;
  \hat p_{ij}^{(k)}=\sigma\!\bigl(R_i^{(k)}-R_j^{(k)}\bigr),
  \label{eq:base-elo}
\end{equation}
where $\sigma(z)=1/(1+e^{-z})$ and $K$ is a global constant. 

\subsection{Adaptive Debiasing Network}
For every pair of answers $(a_i, a_j)$ and a judge $k$, we extract a feature vector $\phi_{ij}^{(k)}$ based on the semantic embeddings of the answers ($\mathbf{e}_i, \mathbf{e}_j$) and the judge's own generated response ($\mathbf{e}_k$). These features, constructed without human labels, capture two key aspects: (i) the semantic and stylistic relationship between the two answers being compared, and (ii) the similarity of each answer to the judge's own output, which serves as a signal for potential self-preference bias. The feature set includes vector operations (e.g., element-wise differences) and scalar metrics (e.g., cosine similarity, KL-divergence). The precise construction of the full feature vector is detailed in Algorithm~\ref{alg:features}.

A small Multilayer Perceptron (MLP) $f_\theta:\mathbb{R}^{6D+9}\rightarrow\mathbb{R}^3$ then maps this feature vector to the adaptive parameters:
\begin{equation}
  \bigl(K_{ij},\;s_i,\;s_j\bigr)=f_\theta\left(\phi_{ij}^{(k)}\right),
\end{equation}
where $K_{ij}$ is the \emph{instance-level} K-factor, and $s_i,s_j\in(0,1)$ are soft labels that replace the hard $r_{ij}^{(k)}$ to mitigate preference noise ($s_i+s_j=1$ after softmax normalization). All parameters $\theta$ are shared across all judges and prompts.

\subsection{The Consensus Anchor Principle}
\label{sec:consensus_principle}
A core challenge in unsupervised debiasing is the absence of a ``golden" reference. Our methodology is built upon a foundational principle: while individual LLM judges are fraught with idiosyncratic biases, the consensus derived from a diverse pool of judges can serve as a powerful proxy for a more stable and reproducible evaluation.

This principle is motivated by our empirical findings (Figure~\ref{fig2}). The biases are not monolithic; they are \textbf{heterogeneous}. For instance, on ArenaHard, a judge like glm-4-flash exhibits strong self-preference (+90\%), while a top-tier model like gpt-4o can paradoxically underrate its own responses ($-$38\%). We hypothesize that these opposing biases tend to partially cancel each other out when aggregated.

\begin{algorithm}[H]
\caption{Feature Vector Construction ($\phi_{ij}^{(k)}$) for UDA}
\label{alg:features}
\begin{algorithmic}[1]
\REQUIRE Answer embeddings $\mathbf{e}_i, \mathbf{e}_j \in \mathbb{R}^{D}$; Judge's own embedding $\mathbf{e}_k \in \mathbb{R}^{D}$
\STATE Let $\odot$ denote the element-wise product.

\STATE \COMMENT{\textbf{Part 1: High-Dimensional Features (6 vectors)}}
\STATE $\mathbf{v}_{\text{diff}, ij} \leftarrow |\mathbf{e}_i - \mathbf{e}_j|$
\STATE $\mathbf{v}_{\text{diff}, ik} \leftarrow |\mathbf{e}_i - \mathbf{e}_k|$ \COMMENT{Judge self-bias feature}
\STATE $\mathbf{v}_{\text{diff}, jk} \leftarrow |\mathbf{e}_j - \mathbf{e}_k|$ \COMMENT{Judge self-bias feature}
\STATE $n_i \leftarrow \|\mathbf{e}_i\|_2, \; n_j \leftarrow \|\mathbf{e}_j\|_2, \; n_k \leftarrow \|\mathbf{e}_k\|_2$
\STATE $\mathbf{v}_{\text{norm-prod}, ij} \leftarrow (\mathbf{e}_i \odot \mathbf{e}_j) / (n_i n_j + \epsilon)$
\STATE $\mathbf{v}_{\text{norm-prod}, ik} \leftarrow (\mathbf{e}_i \odot \mathbf{e}_k) / (n_i n_k + \epsilon)$ \COMMENT{Judge self-bias feature}
\STATE $\mathbf{v}_{\text{norm-prod}, jk} \leftarrow (\mathbf{e}_j \odot \mathbf{e}_k) / (n_j n_k + \epsilon)$ \COMMENT{Judge self-bias feature}

\STATE \COMMENT{\textbf{Part 2: Scalar Features (9 scalars)}}
\STATE $s_{ij} \leftarrow \cos(\mathbf{e}_i, \mathbf{e}_j)$, \; $s_{ik} \leftarrow \cos(\mathbf{e}_i, \mathbf{e}_k)$, \; $s_{jk} \leftarrow \cos(\mathbf{e}_j, \mathbf{e}_k)$
\STATE $\mathbf{p}_i \leftarrow \mathrm{softmax}(\mathbf{e}_i), \; \mathbf{p}_j \leftarrow \mathrm{softmax}(\mathbf{e}_j), \; \mathbf{p}_k \leftarrow \mathrm{softmax}(\mathbf{e}_k)$
\STATE $d_{\text{KL}}^{(j,i)} \leftarrow \mathrm{KL}(\mathbf{p}_j \parallel \mathbf{p}_i)$, \; $d_{\text{KL}}^{(i,k)} \leftarrow \mathrm{KL}(\mathbf{p}_i \parallel \mathbf{p}_k)$, \; $d_{\text{KL}}^{(j,k)} \leftarrow \mathrm{KL}(\mathbf{p}_j \parallel \mathbf{p}_k)$
\STATE $\ell_{\text{diff}, ij} \leftarrow |n_i - n_j|$, \; $\ell_{\text{diff}, ik} \leftarrow |n_i - n_k|$, \; $\ell_{\text{diff}, jk} \leftarrow |n_j - n_k|$

\STATE \COMMENT{\textbf{Part 3: Concatenation}}
\STATE $\phi_{ij}^{(k)} \leftarrow \mathrm{concat}\bigl[
   \mathbf{v}_{\text{diff}, ij}, \mathbf{v}_{\text{diff}, ik}, \mathbf{v}_{\text{diff}, jk}, $
\STATE $\quad \mathbf{v}_{\text{norm-prod}, ij}, \mathbf{v}_{\text{norm-prod}, ik}, \mathbf{v}_{\text{norm-prod}, jk}, $
\STATE $\quad s_{ij}, s_{ik}, s_{jk}, \; d_{\text{KL}}^{(j,i)}, d_{\text{KL}}^{(i,k)}, d_{\text{KL}}^{(j,k)}, $
\STATE $\quad \ell_{\text{diff}, ij}, \ell_{\text{diff}, ik}, \ell_{\text{diff}, jk}
\bigr]$
\RETURN $\phi_{ij}^{(k)} \in \mathbb{R}^{6D+9}$
\end{algorithmic}
\end{algorithm}

However, we acknowledge a critical limitation of this assumption: if a majority of judges share a systemic bias (e.g., a preference for verbosity), their consensus will reflect a ``bias average" rather than a debiased center. In such a scenario, UDA might inadvertently reinforce this majority bias.

Therefore, UDA does not treat the consensus score, denoted as $\mathbf{G}$, as an infallible ground truth. Instead, we define it as a \textbf{stabilizing proxy target}. The objective of UDA is not to perfectly replicate this consensus, but to use it as a supervisory signal to \textbf{reduce system variance}. By training each judge to align with the collective agreement, UDA effectively mitigates the extreme, idiosyncratic preferences of any single judge. The underlying hypothesis is that this variance reduction will enhance the overall reproducibility and alignment with true quality, a hypothesis we validate empirically against human judgments in our experiments.

\begin{table*}[ht]
\centering
\small
\setlength{\tabcolsep}{1mm}

\begin{tabular}{lcccccccccc}
\toprule
Model &
\makecell{gpt-4o} &
\makecell{claude\\-3.5} &
\makecell{glm-4\\-air} &
\makecell{glm-4\\-plus} &
\makecell{glm-4\\-flash} &
\makecell{doubao-1.5pro} &
\makecell{qwen-max} &
\makecell{deepseek-v3} &
\makecell{gemini-\\2.0-flash} &
\makecell{deepseek-r1} \\
\midrule
Baseline Elo  & 162.0 & 154.5 & 122.6 & 99.7 & 118.2 & 170.1 & 109.0 & 149.6 & \textbf{341.9} & 157.5 \\
UDA Method    & \textbf{71.5} & \textbf{44.8} & \textbf{47.8} & \textbf{42.1} & \textbf{50.3} & \textbf{72.6} & \textbf{58.1} & \textbf{61.1} & \textbf{128.8} & \textbf{71.1} \\
\midrule
Reduction\% & 55.8 & 71.0 & 61.1 & 57.8 & 57.4 & 57.3 & 46.7 & 59.2 & 62.3 & 54.9 \\
\bottomrule
\end{tabular}

\caption{Inter-LLM judge score standard deviation ($\downarrow$) on ArenaHard Dataset.}
\label{tab:variance}

\vspace{1.2em}  

\begin{tabular}{lcccccccccc}
\toprule
Model &
\makecell{gpt-4o} &
\makecell{claude\\-3.5} &
\makecell{glm-4\\-air} &
\makecell{glm-4\\-plus} &
\makecell{glm-4\\-flash} &
\makecell{doubao-1.5pro} &
\makecell{qwen-max} &
\makecell{deepseek-v3} &
\makecell{gemini-\\2.0-flash} &
\makecell{deepseek-r1} \\
\midrule
Baseline Elo & 310.9 & 119.3 & 91.3 & 98.6 & 161.5 & 250.1 & 182.0 & 165.6 & 131.1 & 429.1 \\
UDA Method   & \textbf{125.0} & \textbf{28.3} & \textbf{56.3} & \textbf{37.2} & \textbf{80.4} & \textbf{64.8} & \textbf{61.8} & \textbf{77.4} & \textbf{52.7} & \textbf{125.8} \\
\midrule
Reduction\% & 59.8 & 76.3 & 38.3 & 62.2 & 50.2 & 74.1 & 66.0 & 53.3 & 59.8 & 70.7 \\
\bottomrule
\end{tabular}

\caption{Inter-LLM judge score standard deviation ($\downarrow$) on \textsc{Human-Annotated Transfer Set}.}
\label{tab:variance_transfer}
\end{table*}

\subsection{Consensus-Driven Training}
Let the mini-batch $\mathcal{B}_b$ contain \emph{all} pairwise comparisons issued by a single judge $k_b$ in one forward pass. After applying the adaptive update rule, we obtain a vector of per-model Elo scores $\mathbf{R}^{(b)}\!\in\!\mathbb{R}^N$ for the batch. Across $B$ such judge-homogeneous batches, we first compute the consensus anchor using $\bar{\mathbf{R}} = \frac{1}{B}\sum_{b=1}^{B} \mathbf{R}^{(b)}$ and $\mathbf{G} = \frac{1}{B}\sum_{b=1}^{B} \mathbf{R}^{\text{base}}_{(b)}$.

where $\mathbf{R}^{\text{base}}_{(b)}$ are the scores produced by the baseline Elo rule with fixed $K\!=\!32$. We use $\mathbf{G}$ as our optimization target. It is important to clarify that we are aware of the potential biases within the baseline Elo system. However, in a fully unsupervised setting, the consensus score $\mathbf{G}$ represents the most practical and accessible proxy for a stable, judge-agnostic ranking. Our framework is designed to learn a corrective function that maps raw comparisons to debiased outcomes, using the noisy but stable signal from $\mathbf{G}$ as its guide.

The training objective is:
\begin{align}
\mathcal{L}(\theta) ={}& \alpha\sum_b\|\mathbf{R}^{(b)}-\mathbf{G}\|_2^2 \nonumber + \sigma\|\bar{\mathbf{R}}-\mathbf{G}\|_2^2\\
& + \beta\sum_b(1-\rho_p(\mathbf{R}^{(b)},\mathbf{G})) 
\label{eq:loss}
\end{align}
with $\rho_p(\cdot,\cdot)$ the Pearson correlation. Minimising Eq.\,\eqref{eq:loss} simultaneously (i) pulls each judge-specific trajectory toward the baseline consensus, (ii) maximises the linear rank correlation with that consensus, and (iii) enforces collective agreement across all judges. The ultimate success of this approach is evaluated by its ability to improve correlation with external, human-annotated ground truth.

\subsection{Adaptive Elo Updating Rule}
After training, judge $k$ updates its scores with

\begin{algorithm}[H]
\caption{UDA Adaptive Elo Update}
\label{alg:update}
\begin{algorithmic}
\REQUIRE responses $a_i,a_j$, judge $k$, current scores $R_i^{(k)},R_j^{(k)}$
\STATE compute features $\phi_{ij}$ (Algorithm~\ref{alg:features})
\STATE $(K_{ij},s_i,s_j) \gets f_\theta(\phi_{ij})$ \COMMENT{forward pass}
\STATE $\hat p_{ij}^{(k)} \gets \sigma\!\bigl(R_i^{(k)}-R_j^{(k)}\bigr)$
\STATE update
\STATE \quad $R_i^{(k)} \leftarrow R_i^{(k)} + K_{ij}\bigl(s_i - \hat p_{ij}^{(k)}\bigr)$
\STATE \quad $R_j^{(k)} \leftarrow R_j^{(k)} + K_{ij}\bigl(s_j - (1-\hat p_{ij}^{(k)})\bigr)$
\end{algorithmic}
\end{algorithm}

\subsection{Theoretical Motivation for Consensus Alignment}
Our core claim is that aligning diverse judges towards a consensus reduces aggregate system bias. To provide theoretical motivation for this approach, we analyze a simplified, idealized model of the alignment process.

\begin{theorem}[Principle of Aggregate Bias Reduction]
Let $R_i^*$ be the unknown true Elo score for model $i$. Let $R_i^{(k)}$ be the score assigned by judge $k$, which includes a bias term $\epsilon_i^{(k)} = R_i^{(k)} - R_i^*$. The UDA procedure, by optimizing each judge's score towards the consensus, is motivated by the principle that this reduces the total expected absolute bias across all judges. Under an idealized linear shrinkage model of the consensus alignment, for any model $i$, the expected aggregate bias after alignment is less than or equal to the baseline aggregate bias:
$$\sum_{k=1}^{M}\mathbb{E}\bigl[|\epsilon_i^{(k)}(\text{UDA})|\bigr] \le \sum_{k=1}^{M}\mathbb{E}\bigl[|\epsilon_i^{(k)}(\text{base})|\bigr]$$
\end{theorem}

\begin{proof}[Proof Sketch (Illustrative)]
We model the alignment as a direct shrinkage towards the mean, serving as a theoretical motivation for the consensus-driven loss in Eq. \eqref{eq:loss}. Let the updated bias be a convex combination of the original bias and the average bias: $\epsilon_i^{(k)}(\text{UDA}) = \alpha\epsilon_i^{(k)} + (1-\alpha)\bar{\epsilon}_i$.

By applying the triangle inequality and Jensen's inequality, we can show that the sum of absolute biases after this alignment step is bounded by the original sum:
$$\sum_{k=1}^M |\epsilon_i^{(k)}(\text{UDA})| \le \sum_{k=1}^M \left( \alpha|\epsilon_i^{(k)}| + (1-\alpha)|\bar{\epsilon}_i| \right) \le \sum_{k=1}^M |\epsilon_i^{(k)}|$$
This idealized model illustrates the underlying principle: while an individual, well-calibrated judge's accuracy might slightly decrease (a trade-off we observe empirically), the collective variance and aggregate bias across all judges are provably reduced. We emphasize that this is a simplified linear model intended to provide intuition, while the full UDA framework employs a learned non-linear function $f_\theta$ to achieve this alignment adaptively. The full algebraic details are provided in Appendix.
\end{proof}

\begin{figure*}[ht]
\centering
\includegraphics[width=1.8\columnwidth]{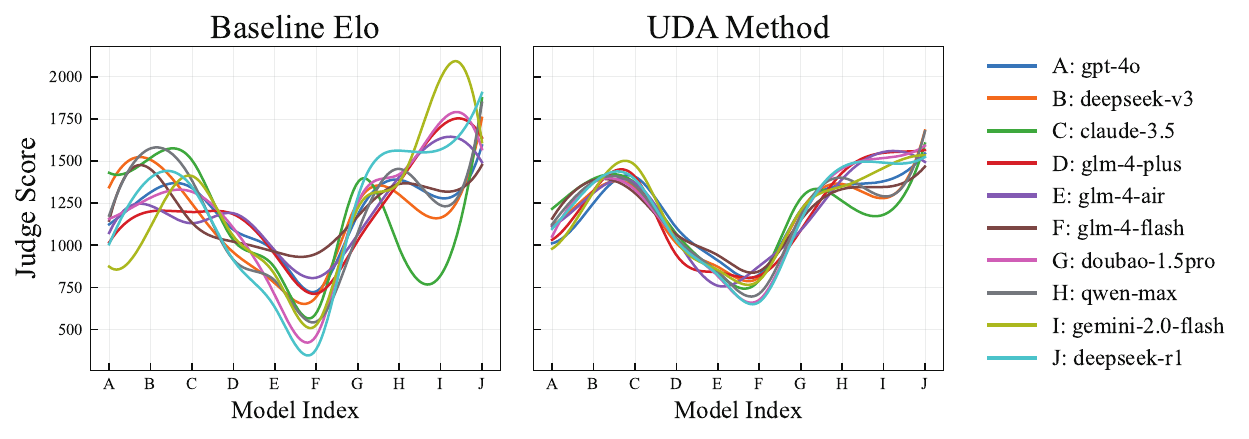}
\caption{Score stability across ten judge llms. Left: Baseline Elo. Right: UDA. Our method markedly aligns scores across diverse LLM judges, yielding significantly lower inter-judge variance.}
\label{fig3}
\end{figure*}

\section{Experiments}
\subsection{Experimental Setup}

\subsubsection{Metrics.}
\textbf{Inter-Judge Standard Deviation:} The standard deviation of a model’s scores across all LLM judges. Lower values signify stronger inter-judge agreement and reduced bias, measuring UDA's success in curbing score dispersion.
\textbf{Pearson Correlation with Human Judgments:} The correlation between a method's generated ranking and the human-annotated ground-truth. Higher values indicate closer alignment with human preferences.
\textbf{Alignment with Baseline Consensus (for analysis):} The Pearson correlation and MSE between a judge’s UDA score and the baseline Elo consensus (mean across judges). These internal metrics verify that our training objective successfully pulls individual judges toward a stable collective anchor.

\textbf{Datasets.} We evaluate on two datasets:
\textbf{ArenaHard}: A subset of 500 prompts from the ArenaHard benchmark, with responses from 10 diverse LLMs (e.g., GPT-4o, DeepSeek-V3). The same 10 models act as judges, yielding 450,000 pairwise judgments.
\textbf{Human-Annotated Transfer Set}: A new dataset of 100 open-ended prompts with human preference labels, designed to test zero-shot generalization. The data pipeline is identical to ArenaHard.

\textbf{Training Details.} The adapter is a 3-layer MLP trained with AdamW. We use a learning rate scheduler and select the best model based on validation loss. Full hyperparameters are in the Appendix.

\textbf{Baselines.} The main baseline is the \textbf{Baseline Elo} with a fixed K-factor ($K=32$). We compare it against our full \textbf{UDA} method. The \textbf{Consensus} score, our optimization target, is evaluated for its correlation with human judgments.

\begin{figure}[ht]
\centering
\includegraphics[width=0.85\linewidth]{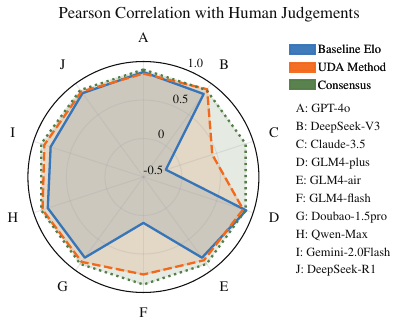}
\caption{Per-model Pearson correlation ($\uparrow$) between the judge scores computed with the baseline method , UDA method and consensus result, respectively, and the human-annotated ground-truth scores. Notably, consensus scores correlate best with human judgments; the uniform correlation of 0.89 arises because the consensus is shared across all judges, irrespective of individual model participation. This validates the consensus as a robust optimization target.}
\label{fig:pearson}
\end{figure}

\begin{figure*}[ht]
\centering
\includegraphics[width=0.8\linewidth]{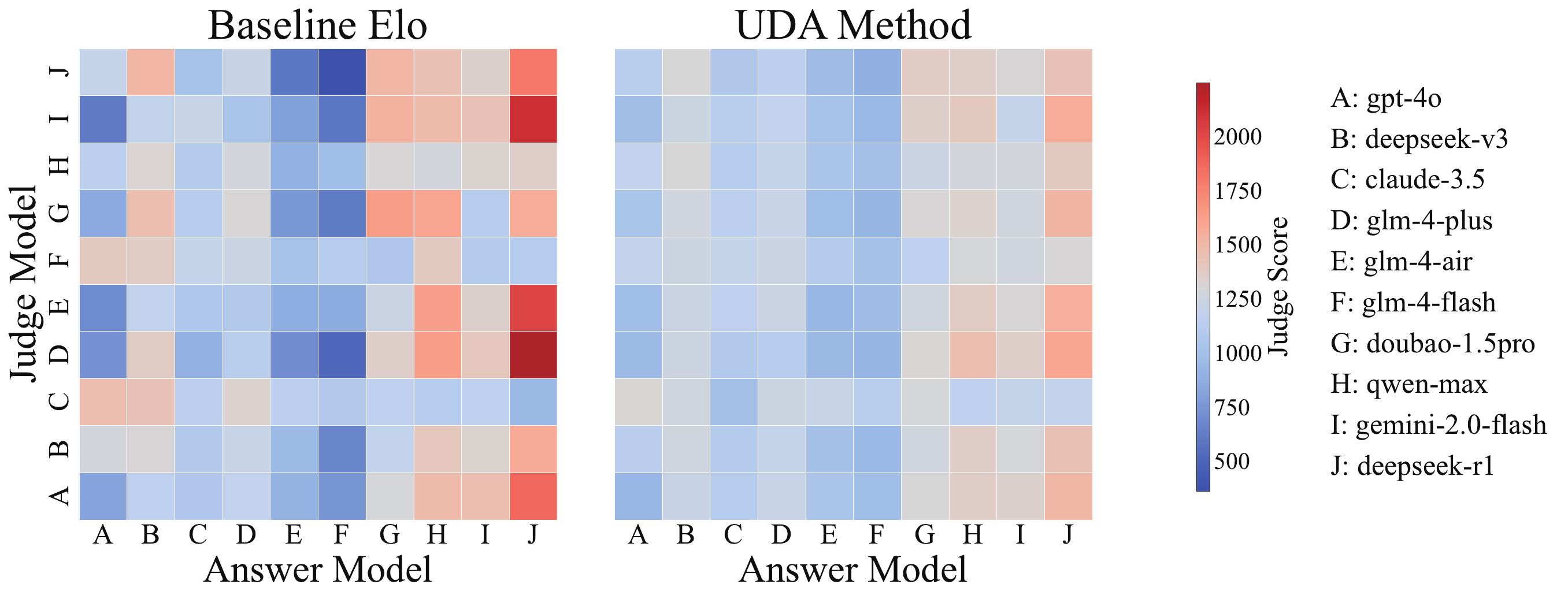}
\caption{Judge-score heat-maps on \textsc{Human-Annotated Transfer Set}. After refinement with UDA, the scores judged by different LLM judges converge markedly, yielding a visibly narrower chromatic variance within each column of the heatmap.}
\label{fig:heatmap}
\end{figure*}

\subsection{Experiment Results}
\subsubsection{Result on ArenaHard}
\paragraph{Cross-Model Variance Reduction}
Table~\ref{tab:variance} reports the per-model standard deviation of Elo scores across ten distinct judge LLMs on ArenaHard. UDA significantly reduces the average inter-judge standard deviation from $158.5$ to $64.8$, a \textbf{relative decrease of $59.1\%$}. The most volatile model, \textit{gemini-2.0-flash}, exhibits the largest absolute shrinkage ($341.9 \rightarrow 128.8$), confirming UDA's efficacy against extreme judge-specific bias.

\paragraph{Visual Consistency}
Figure~\ref{fig3} displays the score trajectories under both methods. With baseline Elo (left), divergent priors yield highly inconsistent curves with score envelopes exceeding $340$ points. UDA (right) compresses this envelope to within $130$ points, demonstrating the emergence of a stable consensus. This consistent reduction in variance verifies that UDA successfully suppresses judge-specific priors, yielding more stable rankings across evaluators.

\subsubsection{Zero-Shot Transfer to Human-Annotated Dataset}
\label{sec:transfer_human}
We applied the UDA adapter trained solely on ArenaHard directly to the \textsc{Human-Annotated Transfer Set} without any retraining to evaluate its zero-shot generalization.

\paragraph{Variance Reduction and Correlation with Human Judgments}
Even on unseen prompts, UDA proves highly effective. As shown in Table~\ref{tab:variance_transfer}, it reduces the average inter-judge standard deviation from $193.95$ to $70.97$, a \textbf{63.4\% relative decrease}. More critically, Figure~\ref{fig:pearson} shows that this increased agreement translates to better alignment with human preferences. UDA (orange markers) consistently outperforms the baseline (blue markers), raising the average Pearson correlation from $0.651$ to $0.812$ (\textbf{+24.7\%}). Notably, it elevates even low-capacity judges (\eg \texttt{glm-4-flash}) to reach near-parity with high-quality ones (\eg \texttt{deepseek-r1}) after debiasing. Furthermore, the consensus result exhibits the strongest human correlation, validating its use as an effective optimization objective.

\paragraph{Qualitative Heatmap Analysis.} Figure\ref{fig:heatmap} offers a qualitative view of the alignment by heatmaps of judge scores. The \textbf{Baseline Elo} map (left) shows pronounced vertical stripes, indicating strong, judge-induced distortions where a single judge rates models differently from its peers. In contrast, the \textbf{UDA Method} (right) yields a much smoother pattern where the columnar structure aligns more closely with human annotations, confirming effective cross-judge calibration and alignment with a more objective standard.

\begin{table}[ht]
\centering
\small
\setlength{\tabcolsep}{4pt}
\begin{tabular}{lcc}
\toprule
\textbf{Method} & \makecell{\textbf{Avg. Inter-Judge} \\ \textbf{Std. Dev. (↓)}} & \makecell{\textbf{Human Pearson} \\ \textbf{Corr. (↑)}} \\
\midrule
Baseline Elo & 193.95 & 0.651 \\
UDA (Full) & 70.97 & \textbf{0.812} \\
UDA (Ablated) & \textbf{65.64} & 0.510 \\
\midrule
Reduction (Full UDA) & 63.4\% & \textbf{+24.7\%} \\
Reduction (Ablated) & \textbf{66.2\%} & -21.6\% \\
\bottomrule
\end{tabular}
\caption{Ablation study results on the \textsc{Human-Annotated Transfer Set}. Removing self-awareness features (UDA (Ablated)) further reduces score variance but critically degrades correlation with human judgments, highlighting their essential role in achieving meaningful alignment.}
\label{tab:ablation}
\end{table}

\subsubsection{Ablation Study: The Critical Role of Self-Awareness Features}
To isolate the contribution of our key innovation—features derived from the judge's own response—we conducted a critical ablation study. We trained and evaluated a variant of UDA, hereafter referred to as \textbf{UDA (Ablated)}, where all features related to the judge's own answer embedding ($\mathbf{e}_k$) were removed from the input vector $\phi_{ij}^{(k)}$. This reduced the feature space dimensionality from 4617 to 1539 and effectively rendered the model blind to the judge's stylistic priors. The model was trained using the same consensus-driven objective.

The results, summarized in Table \ref{tab:ablation}, reveal a fascinating trade-off. On one hand, the ablated model achieved an even greater reduction in inter-judge standard deviation, cutting it by \textbf{66.2\%} (from 193.95 to 65.64). This indicates that, when stripped of self-awareness features, the model becomes exceptionally effective at its direct optimization goal: forcing all judges to agree with each other.

However, this enhanced consensus came at a steep price. The average Pearson correlation with human judgments plummeted from \textbf{0.651} (Baseline) to \textbf{0.510} for the ablated model, a significant decrease. In contrast, the full UDA model, equipped with self-awareness features, boosted this correlation to \textbf{0.812}.

This finding provides powerful evidence that the self-awareness features are not merely beneficial; they are \textbf{essential for grounding the debiasing process}. Without them, the consensus-driven optimization risks converging to a "false consensus"—a state where judges agree with each other more strongly, but their collective judgment diverges from a human-aligned ground truth. In essence, knowing its own stylistic preferences allows a judge (and by extension, the UDA framework) to properly discount them, thereby steering the collective evaluation towards a more objective and meaningful alignment.

\section{Analysis and Discussion}
In this section, we analyze \textbf{UDA} in light of recent advances in LLM-as-a-judge calibration and contrast its empirical behavior with the theoretical claims made in Section \ref{sec:UDA}.

\subsection{Methodological Novelty in Context}
\textbf{UDA}\xspace revises the canonical Elo update in two orthogonal directions that jointly attack \emph{judge-induced inconsistency}.  
First, baseline Elo applies a \emph{global}, \emph{static} $K$, implicitly assuming homoscedastic noise. In practice, this is violated by LLM judges whose priors correlate with their own outputs. \textbf{UDA}\xspace replaces the constant $K$ with an \emph{instance-adaptive} $K_{ij}$ learned from human-free features that capture both inter-answer relationships and self-preference signals. Consequently, likely biased comparisons are down-weighted.

Second, \textbf{UDA}\xspace replaces the raw win label $r_{ij}^{(k)}$ with learned \emph{soft} labels $(s_i,s_j)$ that represent a bias-corrected win probability. This entire procedure is trained unsupervised, guided only by minimizing dispersion among all judges' Elo trajectories to distill a robust latent consensus.

\subsection{Empirical Fidelity}
Our experiments confirm that UDA's consensus-driven objective successfully translates to empirical gains. The training loss (Eq.\,\ref{eq:loss}) guides the model to reduce inter-judge variance while maintaining rank fidelity to a stable consensus anchor. As evidenced by zero-shot transfer results (detailed in Appendix), UDA significantly improves alignment with the baseline consensus across both Pearson correlation and MSE metrics, confirming that the optimization effectively suppresses judge-specific noise.

Crucially, this reduction in variance corresponds to a greater alignment with human evaluators. UDA boosts the average Pearson correlation with human judgments by 24.7\%, democratizing evaluation by elevating weaker judges to parity with stronger ones (Figure \ref{fig:pearson}). The dramatic compression of score variance, as visualized in Figure \ref{fig3}, empirically validates the principle motivating Theorem 1: pulling scores toward a consensus reduces system-wide variance and produces more reproducible, judge-invariant rankings.

\section{Conclusion}
We introduced \textbf{UDA}, an unsupervised, plug-and-play framework that mitigates self-preference bias in LLM-as-a-judge by dynamically calibrating the Elo update rule.  
By learning an instance-level K-factor and soft win-probabilities from lightweight, human-free features, UDA distills a consensus ranking that is simultaneously \emph{judge-invariant} and \emph{human-aligned}.  
Extensive experiments on two benchmarks confirm that UDA (i) shrinks inter-judge score dispersion by 59–63\%, (ii) raises average Pearson correlation with human labels by 25\%, and (iii) elevates low-capacity judges to parity with state-of-the-art models without retraining.  
Theoretical analysis supports the bias reduction under mild assumptions.  
Future work will explore richer feature spaces, weighted consensus objectives, and sub-quadratic approximations to scale beyond hundreds of judges.  
Code and data are publicly released to foster reproducible LLM evaluation.

\section{Acknowledgments}
This work was completed during Yang Zhang's internship at ZhipuAI. We would like to express sincere gratitude to the ZhipuAI 
data annotation team, including Zikang Wang, Xiaohan Zhang and others, for their diligent efforts in curating the human-annotated datasets that made this research possible.


\bibliography{aaai2026}

\appendix
\section{Appendix}

\subsection{Limitations and Future Work}
\paragraph{Feature Sensitivity.}
The current feature vector $\phi_{ij}$ relies heavily on pre-trained BERT embeddings.  
While effective, this design assumes that semantic similarity is a \emph{sufficient} proxy for self-preference.  
Future work could incorporate \emph{instruction-level} features (e.g., task category, prompt difficulty) or \emph{meta-judge} signals (e.g., perplexity under the judge’s own decoding) to better disentangle intrinsic quality from stylistic preference.

\paragraph{Objective Tradeoffs.}
The consensus loss $\mathcal{L}$ treats all judges as \emph{equally informative}, which may \emph{over-correct} high-quality judges like GPT-4o (Figure~\ref{fig:pearson}).  
A weighted variant that down-weights judges with historically poor correlation could preserve \emph{discriminative power} while still reducing bias.

\paragraph{Scalability.}
Training the adapter on 450K pairwise judgments is computationally lightweight (8 hours on a single 4090), but the $O(MN)$ complexity of the consensus loss becomes prohibitive for $M\gg100$.  
Approximate solutions—e.g., clustering judges into \emph{meta-judges}—could retain debiasing benefits at scale.

\subsection{Additional Empirical Analysis}
\label{app:empirical-analysis}
\paragraph{Loss-driven Convergence.}
Recall that the training objective (Eq.\,\ref{eq:loss}) enforces three complementary desiderata: (i) \emph{anchor alignment} via MSE, (ii) \emph{rank fidelity} via Pearson correlation, and (iii) \emph{collective consensus} via global MSE. Table \ref{tab:transfer-cor-mse} shows the zero-shot transfer performance on the \textsc{Human-Annotated Transfer Set}. UDA boosts the mean Pearson correlation with the baseline-Elo consensus from $0.7363$ to $0.8778$ and slashes the average squared distance to the anchor scores from $43,152.6$ to $19,452.7$. This confirms that the loss composition successfully suppresses judge-specific noise while maintaining fidelity to a stable, judge-agnostic ranking.

\begin{table}[ht]
\centering\footnotesize
\setlength{\tabcolsep}{1mm}
\begin{tabular}{l|cc|cc}
\toprule
\multirow{2}{*}{Model} & \multicolumn{2}{c|}{Pearson $\uparrow$} & \multicolumn{2}{c}{MSE $\downarrow$} \\
\cmidrule(lr){2-3}\cmidrule(lr){4-5}
& baseline & UDA & baseline & UDA \\
\midrule
gpt-4o           & 0.9432 & 0.9280 & 14070.3 & 13697.7 \\
deepseek-v3      & 0.9102 & 0.9868 & 11559.5 & 13369.0 \\
claude-3.5       & $-$0.2248 & 0.1791 & 108549.5 & 66849.0 \\
glm4-plus        & 0.9746 & 0.9614 & 59878.7 & 6080.4 \\
glm4-air         & 0.8971 & 0.9627 & 32496.8 & 9160.3 \\
glm4-flash       & 0.1865 & 0.8890 & 72145.9 & 34508.7 \\
doubao-1.5pro    & 0.9066 & 0.9906 & 25488.3 & 8438.5 \\
qwen-max         & 0.8955 & 0.9629 & 19261.0 & 22277.0 \\
gemini-2.0-flash & 0.9222 & 0.9586 & 52760.4 & 9481.8 \\
deepseek-r1      & 0.9518 & 0.9590 & 35315.3 & 10664.2 \\
\midrule
\textbf{Average} & 0.7363 & \textbf{0.8778} & 43152.6 & \textbf{19452.7} \\
\bottomrule
\end{tabular}
\caption{Analysis of Alignment with Baseline Consensus. Zero-shot transfer results on \textsc{Human-Annotated Transfer Set}.  UDA improves average Pearson correlation by $+0.1415$ and reduces average MSE by $23{,}699.9$ without retraining. We adopt the baseline's mean scores as ground-truth supervision.}
\label{tab:transfer-cor-mse}
\end{table}

\subsection{Prompts for Answer Generation and Judging}
\label{app:prompts}

We release \emph{all} prompts used to (i) elicit model answers and (ii) prompt the LLM judges.  
The answer prompt and judge prompt are fixed across models to ensure reproducibility. These prompts originate from our empirical observations of and iterative refinements to both the large language model’s responses and the pairwise evaluation procedure.

\vspace{1mm}
\begin{minipage}{0.97\linewidth}
\begin{tcolorbox}[title={\small Answer Generation Prompt}]
\small
\textbf{System:} You are the most advanced artificial intelligence ever created by humankind, possessing extensive knowledge of the world. Please answer the following question:\\
\textbf{User:} \{prompt\}
\end{tcolorbox}
\end{minipage}

\vspace{2mm}
\begin{minipage}{0.97\linewidth}
\begin{tcolorbox}[title={\small Pairwise Judging Prompt}]
\small
\textbf{System:} Please act as an impartial adjudicator and assess the quality of the two AI assistants’ responses to the user’s query below. Your evaluation should adopt the user’s perspective and consider dimensions such as the assistants’ helpfulness.\\
\textbf{User:} \textbf{Scoring Criteria}
1. Compare the quality of the two responses in relation to the user’s query and provide a concise justification. Avoid positional bias; the order of presentation must not influence your decision.  
2. If the question is a highly specialized, objective factual query that you cannot accurately adjudicate, return [[C]] to prevent misinformation.  
3. If either response contains errors, explicitly identify them. Do not endorse mutually contradictory statements as simultaneously correct (i.e., do not mark opposing claims from A and B as both being right).  
4. Output exactly one of:
   [[A]] – Assistant A’s response is of higher quality.  
   [[B]] – Assistant B’s response is of higher quality.  
   [[C]] – The two responses are of equal quality.\\
\textbf{Additional Notes}
- Begin by checking for factual accuracy; if an error is present, prioritize it in your judgment.  
- Scores must be based on objective criteria, not subjective preference.  
- Response length is irrelevant; evaluate whether the content genuinely addresses the user’s need from the user’s perspective.  
- Evaluate quality strictly from the user’s standpoint.
\textbf{Example}
User question: What is the English term for "Mining transportation and maintenance fees"?  
Assistant A: The English term is “maintenance fee,” which generally denotes the cost required to keep facilities or equipment in good working condition, including repairs, upkeep, and related expenses.  
Assistant B: The correct English term is “Mine transportation maintenance fee.”  
Analysis: Assistant A’s response contains a factual error; the accurate translation is “Mine transportation maintenance fee.” Assistant B’s answer is correct and satisfies the user’s requirement.  
Result: [[B]]\\
\textbf{Return Format}\\
Analysis: …  
Result: [[x]]\\
\textbf{Input}\\
$[\text{User Question -- Start}]$\\
\textbf{question}\\
$[\text{User Question -- End}]$\\
$[\text{Assistant A -- Start}]$\\
\textbf{answer\_i}\\
$[\text{Assistant A -- End}]$\\
$[\text{Assistant B -- Start}]$\\
\textbf{answer\_j}\\
$[\text{Assistant B -- End}]$\\
\textbf{Please evaluate the quality of the two responses according to the scoring criteria above and provide your final judgment.}

\end{tcolorbox}
\end{minipage}

\subsection{Proofs}
\label{app:proofs}

\subsubsection{Complete Proof of Aggregate Bias Reduction}
\label{app:complete-proof}

\paragraph{Mathematical Setup.}
We analyze an idealized mathematical model of the consensus alignment process to formally demonstrate how it reduces aggregate bias. 
Let $\mathcal{J} = \{1, 2, \dots, M\}$ be the set of judges.  
For a given model $m$, let $R_m^*$ be its unknown true Elo score.  
When evaluated by judge $k \in \mathcal{J}$, the resulting score is
\[
R_m^{(k)}.
\]
We define the bias of judge $k$ for model $m$ as
\[
\epsilon_m^{(k)} = R_m^{(k)} - R_m^*.
\]
This bias $\epsilon_m^{(k)}$ can be positive (over-rating, e.g., \textit{glm-4-flash}) or negative (under-rating, e.g., \textit{gpt-4o}), reflecting the self-preference patterns observed in our experiments.

The baseline Elo system produces scores $R_{m,\text{base}}^{(k)}$ with biases $\epsilon_{m,\text{base}}^{(k)}$.  
The UDA framework updates these scores by pulling them towards a consensus.  
Let the post-alignment score for judge $k$ be $R_{m,\text{UDA}}^{(k)}$.  
The objective function in Eq.~\eqref{eq:loss} effectively models the updated score as a shrinkage towards the mean score:
\[
R_{m,\text{UDA}}^{(k)} 
\approx \alpha R_{m,\text{base}}^{(k)} + (1-\alpha)\bar{R}_{m,\text{base}},
\]
where
\[
\bar{R}_{m,\text{base}} = \frac{1}{M}\sum_{j=1}^M R_{m,\text{base}}^{(j)}
\]
is the average (consensus) score across all judges, and $\alpha \in [0,1]$ is a parameter learned by the network that determines the strength of this alignment.  
The closer a judge is to the consensus, the larger the effective $\alpha$.

\paragraph{Post-Alignment Bias.}
The new bias for judge $k$ after UDA is
\[
\epsilon_{m,\text{UDA}}^{(k)} = R_{m,\text{UDA}}^{(k)} - R_m^*.
\]
Substituting $R_{m,\text{base}}^{(k)} = R_m^* + \epsilon_{m,\text{base}}^{(k)}$ and
$\bar{R}_{m,\text{base}} = R_m^* + \bar{\epsilon}_{m,\text{base}}$ (with
$\bar{\epsilon}_{m,\text{base}} = \frac{1}{M}\sum_{j=1}^M \epsilon_{m,\text{base}}^{(j)}$), we obtain
\begin{align*}
\epsilon_{m,\text{UDA}}^{(k)}
&= \alpha\bigl(R_m^* + \epsilon_{m,\text{base}}^{(k)}\bigr)
   + (1-\alpha)\bigl(R_m^* + \bar{\epsilon}_{m,\text{base}}\bigr) - R_m^* \\[2pt]
&= \alpha\epsilon_{m,\text{base}}^{(k)} + (1-\alpha)\bar{\epsilon}_{m,\text{base}}.
\end{align*}

\paragraph{Aggregate Absolute Bias Reduction.}
Let
\[
S_{\text{base}} = \sum_{k=1}^M \bigl|\epsilon_{m,\text{base}}^{(k)}\bigr|,
\qquad
S_{\text{UDA}} = \sum_{k=1}^M \bigl|\epsilon_{m,\text{UDA}}^{(k)}\bigr|
\]
be the total absolute biases before and after UDA.

\begin{theorem}
For any set of biases $\{\epsilon_{m,\text{base}}^{(k)}\}_{k=1}^M$ and for any $\alpha \in [0,1]$, the total absolute bias after consensus alignment does not increase:
\[
S_{\text{UDA}} \le S_{\text{base}}.
\]
\end{theorem}

\begin{proof}
By the triangle inequality on each term of $S_{\text{UDA}}$,
\begin{align*}
\bigl|\epsilon_{m,\text{UDA}}^{(k)}\bigr|
&= \bigl|\alpha\epsilon_{m,\text{base}}^{(k)} + (1-\alpha)\bar{\epsilon}_{m,\text{base}}\bigr| \\[2pt]
&\le \alpha\bigl|\epsilon_{m,\text{base}}^{(k)}\bigr|
       + (1-\alpha)\bigl|\bar{\epsilon}_{m,\text{base}}\bigr|.
\end{align*}
Summing over all judges,
\begin{align*}
S_{\text{UDA}}
&\le \sum_{k=1}^M \Bigl(\alpha\bigl|\epsilon_{m,\text{base}}^{(k)}\bigr|
                         + (1-\alpha)\bigl|\bar{\epsilon}_{m,\text{base}}\bigr|\Bigr) \\[2pt]
&= \alpha S_{\text{base}} + (1-\alpha) M\bigl|\bar{\epsilon}_{m,\text{base}}\bigr|.
\end{align*}
Applying Jensen’s inequality to the convex function $|\cdot|$,
\[
\bigl|\bar{\epsilon}_{m,\text{base}}\bigr|
= \Bigl|\frac{1}{M}\sum_{j=1}^M \epsilon_{m,\text{base}}^{(j)}\Bigr|
\le \frac{1}{M}\sum_{j=1}^M \bigl|\epsilon_{m,\text{base}}^{(j)}\bigr|
= \frac{S_{\text{base}}}{M}.
\]
Substituting this bound,
\[
S_{\text{UDA}}
\le \alpha S_{\text{base}} + (1-\alpha)M\frac{S_{\text{base}}}{M}
= S_{\text{base}}.
\]
Hence, the total absolute bias is not increased.
\end{proof}

\paragraph{Conclusion.}
The UDA procedure is guaranteed to not increase (and in practice, strictly reduce, unless all biases are already identical) the total absolute bias across the system of judges.  
By pulling extreme outliers—both positive and negative—towards a central consensus, it produces a more stable and less biased overall ranking, which is validated by the significant reduction in inter-judge score standard deviation and improved correlation with human judgments observed in our experiments.

\subsubsection{Full judge score Matrices}
\label{app:full-elo}

\paragraph{Notation.}
Each row corresponds to one judge model; each column corresponds to one answer model.
Models are ordered as: \texttt{gpt-4o}, \texttt{deepseek-v3}, \texttt{claude-3.5}, \texttt{glm-4-plus}, \texttt{glm-4-air}, \texttt{glm-4-flash}, \texttt{doubao-1.5pro}, \texttt{qwen-max}, \texttt{gemini-2.0-flash}, \texttt{deepseek-r1}.

\vspace{2mm}
\begin{table*}[ht]
\centering
\scriptsize
\setlength{\tabcolsep}{1mm}
\begin{tabular}{l|rrrrrrrrrr}
\toprule
Judge & {gpt-4o} & {deepseek-v3} & {claude-3.5} & {glm-4-plus} & {glm-4-air} & {glm-4-flash} & {doubao-1.5pro} & {qwen-max} & {gemini-2.0-flash} & {deepseek-r1} \\
\midrule
J1 & 1123.98 & 1314.97 & 1339.64 & 1092.25 & 960.67 & 726.34 & 1180.35 & 1389.14 & 1280.25 & 1592.40 \\
J2 & 1344.56 & 1507.43 & 1245.94 & 957.63 & 774.87 & 693.88 & 1254.53 & 1301.13 & 1163.73 & 1756.31 \\
J3 & 1430.04 & 1519.86 & 1504.93 & 1030.06 & 866.62 & 597.91 & 1373.22 & 983.95 & 820.78 & 1872.64 \\
J4 & 1016.41 & 1201.16 & 1196.69 & 1184.40 & 945.62 & 712.98 & 1027.81 & 1374.29 & 1702.74 & 1637.88 \\
J5 & 1073.32 & 1234.74 & 1130.96 & 1190.02 & 967.96 & 807.96 & 1062.67 & 1405.89 & 1633.85 & 1492.63 \\
J6 & 1146.40 & 1451.35 & 1133.97 & 1022.55 & 960.62 & 949.84 & 1172.75 & 1364.08 & 1321.84 & 1476.60 \\
J7 & 1156.73 & 1282.67 & 1317.72 & 1102.28 & 704.31 & 463.17 & 1250.19 & 1420.78 & 1731.67 & 1570.48 \\
J8 & 1176.31 & 1574.78 & 1384.54 & 913.84 & 790.57 & 546.31 & 1076.23 & 1453.41 & 1237.85 & 1846.17 \\
J9 & 873.45 & 1110.10 & 1411.40 & 1053.64 & 826.41 & 526.93 & 1220.16 & 1374.48 & 1991.27 & 1612.16 \\
J10 & 1006.40 & 1401.21 & 1346.90 & 911.50 & 632.07 & 383.50 & 1283.69 & 1561.84 & 1568.52 & 1904.37 \\
\midrule
UDA & 1011.84 & 1250.88 & 1407.04 & 1103.99 & 911.00 & 825.96 & 1208.18 & 1356.61 & 1379.88 & 1544.61 \\
UDA & 1113.52 & 1323.18 & 1341.04 & 1012.35 & 872.24 & 809.57 & 1203.75 & 1363.79 & 1279.60 & 1680.96 \\
UDA & 1216.48 & 1394.90 & 1370.49 & 1066.91 & 833.52 & 792.63 & 1277.65 & 1262.71 & 1181.19 & 1603.51 \\
UDA & 1032.41 & 1326.42 & 1406.93 & 941.68 & 839.66 & 819.68 & 1088.63 & 1430.47 & 1549.08 & 1565.04 \\
UDA & 1112.89 & 1311.94 & 1370.03 & 1038.88 & 759.82 & 876.91 & 1088.26 & 1392.70 & 1553.80 & 1494.76 \\
UDA & 1158.16 & 1386.46 & 1309.95 & 1062.83 & 941.13 & 846.27 & 1149.40 & 1333.55 & 1344.87 & 1467.37 \\
UDA & 1051.69 & 1373.38 & 1329.12 & 1038.44 & 817.83 & 676.24 & 1148.93 & 1455.43 & 1517.84 & 1591.10 \\
UDA & 1124.66 & 1376.73 & 1354.04 & 1040.37 & 844.24 & 713.35 & 1178.46 & 1400.94 & 1290.47 & 1676.75 \\
UDA & 979.56 & 1321.11 & 1476.07 & 1027.03 & 853.74 & 791.86 & 1207.69 & 1347.74 & 1460.67 & 1534.51 \\
UDA & 1096.99 & 1370.99 & 1392.02 & 1028.23 & 824.30 & 663.20 & 1146.43 & 1463.64 & 1490.27 & 1523.94 \\
\bottomrule
\end{tabular}
\caption{\textbf{ArenaHard Dataset: baseline Elo (top) vs.\ UDA Method (bottom)}}
\label{tab:arena-matrix}
\end{table*}

\vspace{3mm}

\begin{table*}[ht]
\centering
\scriptsize
\setlength{\tabcolsep}{1mm}
\begin{tabular}{l|rrrrrrrrrr}
\toprule
Judge & {gpt-4o} & {deepseek-v3} & {claude-3.5} & {glm-4-plus} & {glm-4-air} & {glm-4-flash} & {doubao-1.5pro} & {qwen-max} & {gemini-2.0-flash} & {deepseek-r1} \\
\midrule
J1 & 820.93 & 1169.77 & 1063.86 & 1180.98 & 912.45 & 746.43 & 1289.74 & 1484.22 & 1454.52 & 1877.10 \\
J2 & 1281.40 & 1309.34 & 1078.76 & 1210.14 & 959.14 & 664.53 & 1190.17 & 1402.90 & 1326.95 & 1576.69 \\
J3 & 1459.42 & 1434.81 & 1148.82 & 1341.07 & 1154.05 & 1082.81 & 1158.53 & 1110.92 & 1158.02 & 951.53 \\
J4 & 727.93 & 1375.24 & 895.98 & 1122.83 & 710.39 & 523.17 & 1350.34 & 1641.94 & 1405.05 & 2247.13 \\
J5 & 698.44 & 1184.78 & 1060.97 & 1080.97 & 876.92 & 866.45 & 1233.04 & 1633.95 & 1344.93 & 2019.54 \\
J6 & 1387.21 & 1374.16 & 1198.25 & 1235.12 & 1020.42 & 1117.67 & 1061.98 & 1388.73 & 1103.78 & 1112.69 \\
J7 & 852.40 & 1452.30 & 1114.15 & 1309.11 & 755.75 & 613.18 & 1636.28 & 1601.80 & 1106.17 & 1558.86 \\
J8 & 1143.36 & 1324.79 & 1097.86 & 1280.63 & 893.90 & 976.86 & 1311.58 & 1276.34 & 1332.57 & 1362.10 \\
J9 & 606.89 & 1191.82 & 1212.54 & 1035.90 & 797.55 & 595.32 & 1532.11 & 1478.94 & 1432.74 & 2116.16 \\
J10 & 1197.16 & 1508.89 & 1020.43 & 1220.23 & 591.70 & 361.67 & 1506.17 & 1442.62 & 1346.96 & 1804.17 \\
\midrule
UDA & 929.83 & 1221.50 & 1105.68 & 1212.09 & 1044.76 & 976.05 & 1297.16 & 1364.22 & 1348.20 & 1500.51 \\
UDA & 1135.44 & 1266.64 & 1095.96 & 1199.45 & 998.51 & 948.20 & 1264.47 & 1370.28 & 1283.76 & 1437.28 \\
UDA & 1312.28 & 1272.16 & 992.62 & 1238.63 & 1208.68 & 1126.88 & 1290.37 & 1170.10 & 1199.25 & 1189.03 \\
UDA & 958.25 & 1247.67 & 1080.86 & 1118.82 & 953.30 & 923.91 & 1318.88 & 1452.59 & 1350.77 & 1594.96 \\
UDA & 976.49 & 1227.70 & 1160.61 & 1223.52 & 942.34 & 969.75 & 1264.26 & 1381.11 & 1306.87 & 1547.33 \\
UDA & 1191.71 & 1246.18 & 1194.63 & 1231.78 & 1102.26 & 1004.05 & 1162.83 & 1293.45 & 1265.00 & 1307.13 \\
UDA & 1030.98 & 1272.61 & 1156.56 & 1208.91 & 975.73 & 925.70 & 1311.05 & 1334.97 & 1266.51 & 1516.99 \\
UDA & 1179.94 & 1300.77 & 1111.39 & 1197.34 & 1037.00 & 989.38 & 1234.67 & 1280.93 & 1279.26 & 1389.31 \\
UDA & 986.95 & 1239.54 & 1121.67 & 1179.81 & 1017.57 & 946.45 & 1354.42 & 1394.33 & 1197.81 & 1561.45 \\
UDA & 1124.61 & 1303.36 & 1069.66 & 1151.73 & 971.80 & 889.36 & 1381.52 & 1362.55 & 1311.55 & 1433.87 \\
\midrule
Human & 1049.80 & 1284.72 & 663.97 & 1386.02 & 927.27 & 717.69 & 1329.97 & 1485.66 & 1315.73 & 1839.18 \\
\bottomrule
\end{tabular}
\caption{\textbf{Human-Annotated Transfer Set: baseline Elo (top), UDA Method (mid), and Human score (bottom)}}
\label{tab:human-matrix}
\end{table*}

\subsubsection{Full Judge Score Matrices for Ablation Study}
\label{app:ablation-elo}
We provide the complete judge-score matrices from the ablation study conducted on the \textsc{Human-Annotated Transfer Set}. Each row corresponds to a judge, and each column corresponds to a candidate model. The models are ordered as listed in the main paper.

\vspace{2mm}
\begin{table*}[ht]
\centering
\scriptsize
\setlength{\tabcolsep}{0.5mm}
\begin{tabular}{l|rrrrrrrrrr}
\toprule
\textbf{Judge} & \texttt{gpt-4o} & \texttt{d-seek-v3} & \texttt{claude-3.5} & \texttt{glm4-plus} & \texttt{glm4-air} & \texttt{glm4-flash} & \texttt{doubao-pro} & \texttt{qwen-max} & \texttt{gemini-flash} & \texttt{d-seek-r1} \\
\midrule
\multicolumn{11}{c}{\textbf{Baseline Elo Scores}} \\
\midrule
J1 & 820.93 & 1169.77 & 1063.86 & 1180.98 & 912.45 & 746.43 & 1289.74 & 1484.22 & 1454.52 & 1877.10 \\
J2 & 1281.40 & 1309.34 & 1078.76 & 1210.14 & 959.14 & 664.53 & 1190.17 & 1402.90 & 1326.95 & 1576.69 \\
J3 & 1459.42 & 1434.81 & 1148.82 & 1341.07 & 1154.05 & 1082.81 & 1158.53 & 1110.92 & 1158.02 & 951.53 \\
J4 & 727.93 & 1375.24 & 895.98 & 1122.83 & 710.39 & 523.17 & 1350.34 & 1641.94 & 1405.05 & 2247.13 \\
J5 & 698.44 & 1184.78 & 1060.97 & 1080.97 & 876.92 & 866.45 & 1233.04 & 1633.95 & 1344.93 & 2019.54 \\
J6 & 1387.21 & 1374.16 & 1198.25 & 1235.12 & 1020.42 & 1117.67 & 1061.98 & 1388.73 & 1103.78 & 1112.69 \\
J7 & 852.40 & 1452.30 & 1114.15 & 1309.11 & 755.75 & 613.18 & 1636.28 & 1601.80 & 1106.17 & 1558.86 \\
J8 & 1143.36 & 1324.79 & 1097.86 & 1280.63 & 893.90 & 976.86 & 1311.58 & 1276.34 & 1332.57 & 1362.10 \\
J9 & 606.89 & 1191.82 & 1212.54 & 1035.90 & 797.55 & 595.32 & 1532.11 & 1478.94 & 1432.74 & 2116.16 \\
J10 & 1197.16 & 1508.89 & 1020.43 & 1220.23 & 591.70 & 361.67 & 1506.17 & 1442.62 & 1346.96 & 1804.17 \\
\midrule
\multicolumn{11}{c}{\textbf{UDA (Full) Scores}} \\
\midrule
J1 & 1141.36 & 1202.67 & 1138.28 & 1185.60 & 1107.94 & 1116.15 & 1196.94 & 1255.72 & 1311.88 & 1343.46 \\
J2 & 1336.81 & 1233.62 & 1150.74 & 1161.55 & 1110.13 & 1057.50 & 1180.80 & 1192.78 & 1299.93 & 1276.15 \\
J3 & 1377.93 & 1280.68 & 1105.90 & 1268.59 & 1148.17 & 1192.59 & 1221.75 & 1141.70 & 1215.67 & 1047.01 \\
J4 & 1196.71 & 1175.79 & 1084.40 & 1183.14 & 1079.98 & 1038.87 & 1235.12 & 1290.29 & 1267.97 & 1447.74 \\
J5 & 1076.92 & 1195.13 & 1152.85 & 1157.65 & 1079.05 & 1142.77 & 1217.09 & 1291.34 & 1271.03 & 1416.17 \\
J6 & 1305.32 & 1283.27 & 1193.84 & 1197.23 & 1148.34 & 1184.03 & 1158.19 & 1209.66 & 1211.69 & 1108.43 \\
J7 & 1144.90 & 1250.18 & 1200.55 & 1210.39 & 1083.73 & 1097.62 & 1371.63 & 1232.65 & 1094.98 & 1313.39 \\
J8 & 1197.07 & 1237.31 & 1183.12 & 1190.55 & 1068.54 & 1157.58 & 1232.84 & 1194.02 & 1335.07 & 1203.90 \\
J9 & 1100.04 & 1165.76 & 1212.91 & 1158.41 & 1160.29 & 1064.05 & 1252.41 & 1262.05 & 1136.86 & 1487.23 \\
J10 & 1241.42 & 1252.49 & 1124.28 & 1156.55 & 1034.83 & 1007.34 & 1323.65 & 1243.34 & 1258.30 & 1357.81 \\
\midrule
\multicolumn{11}{c}{\textbf{UDA (Ablated) Scores -- w/o Self-Features}} \\
\midrule
J1 & 1102.51 & 1175.33 & 1150.11 & 1158.92 & 1021.45 & 998.64 & 1201.20 & 1308.81 & 1244.17 & 1441.86 \\
J2 & 1188.19 & 1210.87 & 1162.30 & 1182.25 & 1055.78 & 1010.99 & 1195.81 & 1300.15 & 1230.90 & 1403.76 \\
J3 & 1235.40 & 1244.11 & 1120.57 & 1211.70 & 1088.31 & 1090.25 & 1188.93 & 1255.44 & 1222.01 & 1313.28 \\
J4 & 1080.15 & 1221.90 & 1098.88 & 1140.33 & 999.50 & 965.87 & 1210.74 & 1340.22 & 1250.60 & 1511.81 \\
J5 & 1055.77 & 1180.25 & 1145.92 & 1133.19 & 1011.66 & 1022.40 & 1205.11 & 1335.98 & 1240.15 & 1469.57 \\
J6 & 1210.93 & 1233.50 & 1166.14 & 1190.88 & 1066.02 & 1088.71 & 1180.55 & 1280.60 & 1220.31 & 1322.36 \\
J7 & 1100.28 & 1255.12 & 1177.36 & 1195.40 & 1015.99 & 988.54 & 1245.33 & 1322.71 & 1199.88 & 1449.39 \\
J8 & 1155.61 & 1201.78 & 1160.01 & 1177.66 & 1033.47 & 1060.11 & 1199.95 & 1277.30 & 1255.13 & 1378.98 \\
J9 & 1044.89 & 1177.01 & 1188.45 & 1135.22 & 1050.18 & 990.56 & 1225.10 & 1318.88 & 1215.99 & 1453.72 \\
J10 & 1150.36 & 1241.92 & 1118.04 & 1165.01 & 1010.96 & 948.79 & 1228.69 & 1315.52 & 1255.71 & 1414.00 \\
\bottomrule
\end{tabular}
\caption{Full judge-score matrices for Baseline Elo, the full UDA method, and the ablated UDA variant on the \textsc{Human-Annotated Transfer Set}. The ablated model shows visibly lower column-wise variance but diverges more from human-aligned rankings.}
\label{tab:ablation-matrix}
\end{table*}

\paragraph{Reproducibility Checklist.}  
Code, model checkpoints and the \textsc{Human-Annotated Transfer Set} are released at
\begin{center}
\url{https://anonymous.4open.science/r/62AB93CD-23B4}
\end{center}

\end{document}